\documentclass{article}
\usepackage{ijcai22}
\usepackage{xcolor}
\usepackage{times}
\usepackage{soul}
\usepackage{url}
\usepackage[hidelinks]{hyperref}
\usepackage[utf8]{inputenc}
\usepackage{caption}
\usepackage{graphicx}
\usepackage{amsmath}
\usepackage{amsthm}
\usepackage{booktabs}
\usepackage{algorithm}
\usepackage{algorithmic}
\usepackage{amssymb,amsmath,amsthm,dsfont,bm,mathtools,bbm}
\usepackage{pgfplots}
\usepackage{verbatim}
\usepackage[font=small]{caption}
\tikzstyle{basicNode} = [circle, draw=black, fill=white!30]
\tikzstyle{arrow} = [thick,->,>=stealth]
\tikzstyle{transNode} = [circle, draw=none, fill=white!30]
\usepackage{thmtools} 
\usepackage{thm-restate}
\usepackage{tabularx}
\usepackage{multirow}
\usepackage{subfig}
\DeclareMathOperator*{\argmax}{arg\,max}

\usepackage{color,graphicx}
\usepackage{hyperref,xcolor}
\usepackage[inline]{enumitem}
\usepackage[textsize=tiny, textwidth=1.05in, colorinlistoftodos]{todonotes}

\newcommand{\ALG}[1]{ADVISER}
\title{ADVISER: AI-Driven Vaccination Intervention Optimiser for Increasing Vaccine Uptake in Nigeria}
\author{
Vineet Nair$^1$\and
Kritika Prakash$^1$\and
Michael Wilbur$^2$\and
Aparna Taneja$^1$\and\\
Corinne Namblard$^3$\and
Oyindamola Adeyemo$^3$\and
Abhishek Dubey$^2$\and
Abiodun Adereni$^3$\and\\
Milind Tambe$^1$\and
Ayan Mukhopadhyay$^2$\footnote{Corresponding Author}
\affiliations
$^1$Google Research, India\\
$^2$Vanderbilt University, USA\\
$^3$HelpMum, Nigeria\\
\emails
\{vineetn, kritikaprakash, aparnataneja, milindtambe\}@google.com,\\
\{michael.p.wilbur, abhishek.dubey, ayan.mukhopadhyay\}@vanderbilt.edu,\\
corinne.namblard@gmail.com,
\{oyindamola, biodun\}@helpmum.org
}

\begin{document}
\maketitle
\begin{abstract}
More than 5 million children under five years die from largely preventable or treatable medical conditions every year, with an overwhelmingly large proportion of deaths occurring in under-developed countries with low vaccination uptake. One of the United Nations' sustainable development goals (SDG 3) aims to end preventable deaths of newborns and children under five years of age. We focus on Nigeria, where the rate of infant mortality is appalling. We collaborate with HelpMum, a large non-profit organization in Nigeria to design and optimize the allocation of heterogeneous health interventions under uncertainty to increase vaccination uptake, the first such collaboration in Nigeria. Our framework, ADVISER: AI-Driven Vaccination Intervention Optimiser, is based on an integer linear program that seeks to maximize the cumulative probability of successful vaccination. Our optimization formulation is intractable in practice. We present a heuristic approach that enables us to solve the problem for real-world use-cases. We also present theoretical bounds for the heuristic method. Finally, we show that the proposed approach outperforms baseline methods in terms of vaccination uptake through experimental evaluation. HelpMum is currently planning a pilot program based on our approach to be deployed in the largest city of Nigeria, which would be the first deployment of an AI-driven vaccination uptake program in the country and hopefully, pave the way for other data-driven programs to improve health outcomes in Nigeria.
\end{abstract}

\section{Introduction}

% \textcolor{red}{Ayan: Margin enlarged for to-do comments. Remove before submission.}

The state of maternal and infant health in Nigeria is appalling. The estimated maternal mortality rate in Nigeria is about 814 per 100,000 live births; in comparison, Poland and Italy's maternal mortality rate is 2 deaths per 100,000 live births. In fact, Nigeria alone accounts for more than \textbf{10\% of maternal deaths globally}, while only accounting for 2.6\% of the world's population~\cite{whoNigeriaDeaths}. Infant deaths in the country are also shockingly high---\textbf{Nigeria loses 2300 children under five years of age daily}~\cite{okwuwa2020infant}. The sustainable development goals (SDG 1 and SDG 3) aim to mobilize resources to the developing world to address inequity due to poverty and end preventable deaths of infants completely~\cite{un_sdg}. However, we are far from achieving these goals.

In collaboration with HelpMum, a large non-profit organization based in Nigeria, we identify three significant challenges contributing to high mortality rates among mothers and infants. \textbf{First}, with an immunization rate of 13\% for children between 12-23 months, Nigeria has the lowest vaccination rate in Africa. While vaccination is available for free in Nigeria, lack of awareness about the importance of vaccination is one of the major concerns for the low uptake of vaccination.
%On a cursory glance, this is counter-intuitive since vaccinations in Nigeria are available for free. However, lack of awareness about the importance of vaccinations is a primary reason for the low uptake of vaccination. 
\textbf{Second}, HelpMum identified that a primary driver for mothers not taking their children for vaccination is the high transportation cost relative to their income; we found that 46\% of families analyzed as a part of this study earned less than \$25 per month. \textbf{Third}, although several organizations, such as HelpMum, strive to design interventions for at-risk mothers and children, there is a gross imbalance between resource availability and demand for healthcare services. 

\begin{figure*}%
    \centering
    \subfloat[\centering]{{\includegraphics[width=4.5cm]{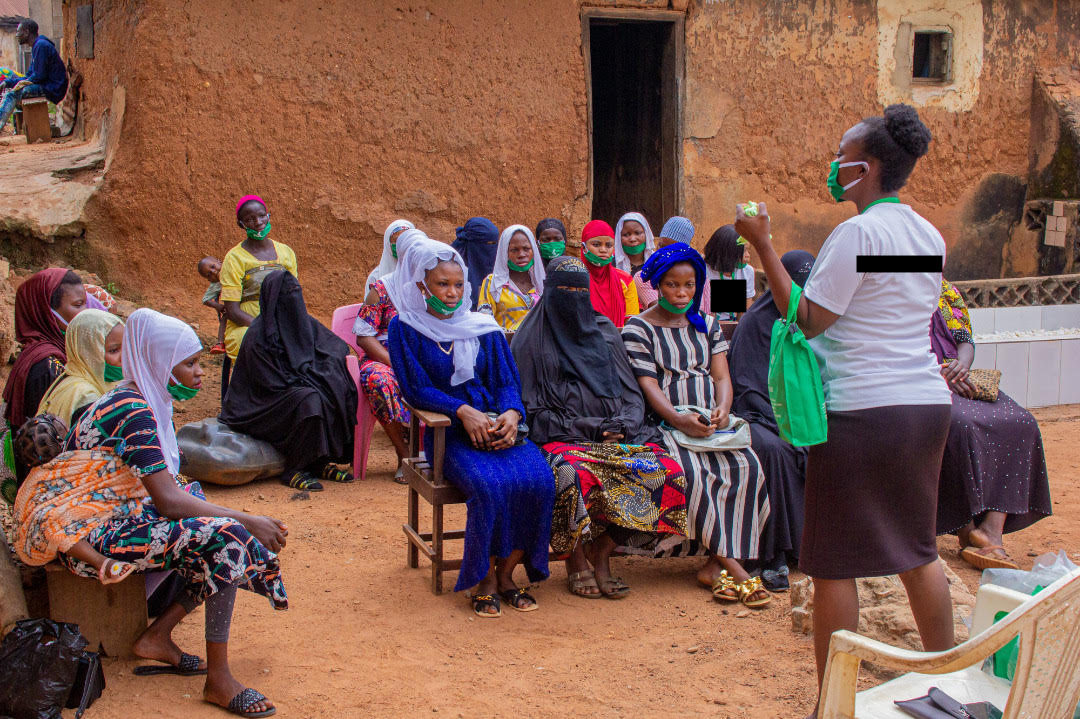}}}%
    \subfloat[\centering]{{\includegraphics[width=13cm]{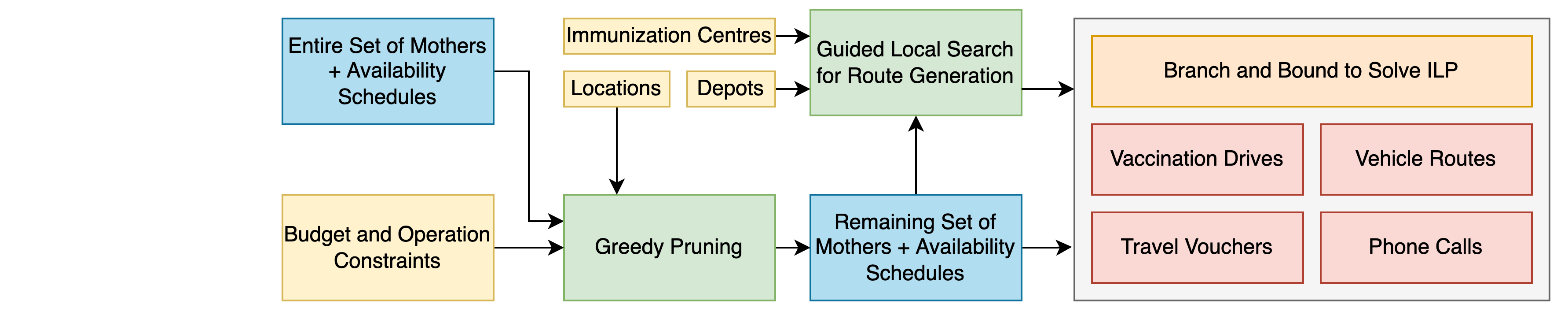}}}%
    \caption{(a) HelpMum reaches out to low-income neighborhoods to distribute clean birth kits and increase awareness of getting vaccinated. (b) A high-level overview of \textbf{ADVISER}: \textbf{A}I \textbf{D}riven \textbf{V}accination \textbf{I}ntervention Optimi\textbf{SER}. We formulate the allocation of heterogeneous resources as an integer linear program (ILP). We use a greedy pruning strategy to make the ILP tractable and use guided local search to generate promising vehicle routes.}%
    \label{fig:helpmumAndWorkflow}%
\end{figure*}
\vspace{-0.5em}
% \begin{figure}
%     \centering
%     \includegraphics[scale=0.25]{imgs/dashboard.png}
%     \caption{Dashboard \textcolor{red}{Ayan: we need to remove the HelpMum logo from the image.}}
%     \label{fig:dashboard}
% \end{figure}
HelpMum works closely with several local and state governments in Nigeria. Even though it uses a vaccination tracking system to remind mothers of upcoming vaccination, it has proven to be ineffective in practice. As part of this project, four new interventions were designed to increase vaccination uptake in Nigeria. HelpMum (including its advisory board) and domain experts guided the design of each intervention. The interventions (described later in the paper) are geared towards increasing the awareness of vaccination, reminding mothers about upcoming vaccinations for their children, providing accessibility to vaccination centers by operating a pick-up and drop-off service, and conducting a door-to-door vaccination delivery program. However, matching the interventions to specific individuals presents a challenge---there exist far too many eligible recipients for the interventions as compared to the available resources. Indeed, the state in Nigeria where HelpMum is based has approximately 1.7 million children under five years of age. In contrast, HelpMum has only $4$ buses at their disposal for picking up mothers, and only a limited number of healthcare workers are available for conducting door-to-door interventions. Moreover, the interventions do not necessarily guarantee successful vaccination. For example, HelpMum has observed that, at times, parents fail to bring their children for immunization despite repeated phone calls about upcoming vaccination schedules. These domain-specific challenges require that the allocation of limited resources is optimized under uncertainty of the outcomes.

% \textbf{First}, we identify that it is imperative to reach out to mothers and educate them about vaccination programs and remind them about vaccination schedules. While HelpMum already performs this intervention through phone calls to all mothers whose children have upcoming vaccination, we seek to add a step of targeted phone calls to mothers who are likely not to present their children for vaccination. \textbf{Second}, HelpMum plans to operate small buses to pick up mothers and drive them to vaccination centers. \textbf{Third}, since there are a limited number of vans that HelpMum can use, it plans to provide travel vouchers to mothers to take their children to health centers for vaccination. \textbf{Fourth}, HelpMum is in the process of establishing a partnership with the local government to enable health workers to travel to neighborhoods with vaccines to perform a door-to-door vaccine delivery service.

We present a principled framework for optimizing the allocation of heterogeneous health interventions under uncertainty. Our approach, named \textbf{\ALG}: \textbf{A}I \textbf{D}riven \textbf{V}accination \textbf{I}ntervention Optimi\textbf{SER}, can guide non-profit organizations and government agencies to increase vaccination uptake in resource-constrained geographies, that are crucial to achieving key goals outlines in SDG 1 and SDG 3 of the United Nations~\cite{un_sdg}. Our approach is based on formulating an integer linear program (ILP) to maximize the cumulative probability of success of heterogeneous interventions under uncertainty. However, in our setting, it is infeasible to solve the ILP directly (even if the inputs were generated). In this paper, we tackle the challenges in a principled manner. 
% Specifically, we formulate an integer linear program (ILP) to maximize the cumulative probability of successful vaccination under domain-specific constraints. However, the ILP is non-trivial to solve due to several challenges. \textbf{First}, directly solving the ILP is infeasible in practice.
% % our formulation consists of over $1.16\times 10^{9}$ decision variables and $3.45\times 10^9$ constraints. 
% \textbf{Second}, the ILP must receive as inputs a set of feasible routes that the vans use based on operational constraints (e.g., constraints on pick-up and drop-off times).
% % constraints, e.g., HelpMum estimates that long waiting times at health centers mean that mothers must be dropped off before 11 a.m. 
% However, the set of feasible routes ($10^{200}$) makes route generation intractable in practice. \textbf{Third}, the probability of success for each intervention must be estimated before optimizing the allocation of resources. 
Specifically, we make the following contributions:
\begin{enumerate*}[label=\textbf{\arabic*)}]
    \item We present a formulation for optimizing the allocation of heterogeneous health resources under uncertainty.
    \item We present a heuristic approach to prune the decision space of the ILP by leveraging the structure of the problem. We also present a theoretical bound on the objective value attainable by the heuristic with respect to the optimal solution to the ILP.  
    \item We show how guided local search can be used to generate promising vehicle routes based on the probability of specific individuals \textit{requiring} the pickup service for vaccination.
    \item We estimate the success of interventions through historical data and community surveys.
    \item We test our algorithmic approach using the data collected by HelpMum. Experimental results demonstrate that the proposed approach significantly outperforms baseline approaches. 
    \item Finally, HelpMum is currently developing a pilot plan to deploy ADVISER in the largest city in West Africa. To the best of our knowledge, our solution would be the first AI-enabled program for increasing vaccination uptake in Nigeria.
\end{enumerate*}
\section{Related Work}

We discuss prior work related to combinatorial resource allocation under uncertainty for achieving health outcomes. The optimization of resources can be either done in a \textit{single-shot} manner or by considering the sequential nature of the decision-making problem. The specific paradigm for resource allocation depends on the specific problem domain. For example, the optimization of patient admissions in hospitals~\cite{hulshof2013tactical}, allocating home healthcare services~\cite{aiane2015new}, and redistribution of patients among hospitals in case of a surge in demand (e.g., in case of a pandemic)~\cite{parker2020optimal} have been modeled as single-shot optimization problems. Specifically, \citeauthor{parker2020optimal} solve the problem of finding optimal demand and resource transfers to minimize the surge capacity and resource shortage during a period of heightened demand due to COVID-19~\cite{parker2020optimal}. \citeauthor{aiane2015new} work on allocating services such as medical, paramedical, and social services delivered to patients in their homes modeled using an MILP formulation~\cite{aiane2015new}. Our problem setting and formulation is most similar to theirs in principle; but the problem setting considered by \citeauthor{aiane2015new} only accounts for travel times by resources as part of the objective and does not account for uncertainty in the outcomes after resource allocation~\cite{aiane2015new}. Moreover, their approach is not scalable to our setting; the number of decision variables and constraints in our problem is $10^5$ times higher. 

Prior work has also explored performing sequential decision-making in the context of resource allocation in healthcare settings. For example,  \citeauthor{mate2021field}~\cite{mate2021field} and \citeauthor{nishtala2021selective}~\cite{mate2021field} model the allocation of targeted phone calls as a restless multi-armed bandit problem (RMAB). In such an approach, historical data is used to estimate the effect of interventions (similar to our approach). Then, the RMAB model is used for planning interventions over multiple decision epochs with limited resources.\citeauthor{tsoukalas2015data} present a data-driven probabilistic framework for clinical decision support by using partially observable Markov decision processes (POMDP)~\cite{tsoukalas2015data}. The POMDP model is based on clinical practice, expert knowledge and data representations in emergency healthcare settings.

\section{Problem Formulation}\label{sec: problem form}

% We formulate the problem of optimizing the intervention allocation as an integer linear program (ILP).
\textbf{Problem Setting:} Our problem setting involves resource allocation to $M$ individuals (e.g., mothers) over $T$ days. We use $[M]$ and $[T]$ as shorthand for $[1,\dots ,M]$ and $[1,\dots ,T]$ respectively. While our goal is to ensure that children get vaccinated, mothers typically take their children for vaccination in our geographic area of interest. As a result, we say that the interventions are designed for mothers. We assume that each mother is eligible for an intervention for a fixed number of contiguous days within these $T$ days (depending on the last date when her child was vaccinated).
% Naturally, this period of eligibility varies across the mothers.
The binary variable $a_{mt}$ denotes whether mother $m\in [M]$ is eligible at time $t \in [T]$; $a_{mt}=1$ if and only if the mother is eligible on day $t$, and is $0$ otherwise. In order to get vaccinated, mothers can either travel to designated health centers, or healthcare officials can visit a mother's house. We divide the region of interest into a grid $G$ consisting of equally sized cells. Each mother's residence and each health center therefore map to unique cells in $G$. We use $d_{mg}$ to denote the distance of mother $m$'s residence from cell $g \in G$.

\noindent \textbf{Interventions:} In collaboration with HelpMum and domain experts, we design four new interventions:
\begin{enumerate*}[label=\textbf{\arabic*)}]
    \item \textit{Phone call}: A phone call is made to the mother reminding her about upcoming vaccination. We denote this intervention by $i_c$.
    \item \textit{Travel Voucher}: A travel voucher is provided to the mother to commute to vaccination centers. We denote this intervention by $i_t$. 
    \item \textit{Bus Pickup}: A bus can pick up a mother (and her child) from her residence and drop them at a vaccination center. Each bus has a capacity of $\gamma_{\ell}$ (for ease of exposition, we assume that $\gamma_{\ell}$ denotes the number of mothers that a bus can accommodate with their children).  
    We denote this intervention by $i_{\ell}$. $F$ denotes the set of buses.
    %Naturally, mothers living far away from the chosen route cannot be picked by such a bus.
    \item \textit{Vaccine Drive}: A health worker goes to a designated locality and vaccinates mothers (children) living nearby who are eligible for vaccination. Naturally, there is a cap on the number of vaccinations a health worker can provide in a day. We denote this cap by $\gamma_v$, and denote this intervention by $i_v$.
\end{enumerate*}\\
We use $I$ to denote the set of interventions, and 
for notational convenience, add no-intervention/empty-intervention, denoted as $i_n$, to this set. HelpMum considers $i_v$ to be highly effective in practice, followed by $i_b$, $i_t$ and $i_c$ (in decreasing order of effectiveness). Each intervention has a cost associated with it; we use $e_j$ to denote the cost associated with intervention $i_j \in I$. Naturally $e_n = 0$ (the cost of no-intervention is $0$), and $e_{\ell} > e_v \gg e_t > e_c$. In particular, employing a bus pickup or the cost of conducting a vaccine drive is relatively much more expensive than giving a travel voucher or making a phone call to a single mother. 
% Note that the costs do not represent the amount of money spent per mother; rather, they denote the total cost for an intervention.

\noindent \textbf{Outcomes:} Let $p_{mj}$ be the probability of mother $m$ taking her child for vaccination given intervention $i_j \in I$. 
% We describe the estimation of these probabilities in section~\ref{subsec: param est}.
%\textcolor{red}{We remark here that $p_{m,j}$ is different from the conditional probability, and is equal to the probability after making an \textit{intervention}, computed using
%the \textit{do-calculus}~\cite{Pearl16}. We discuss how these probabilities are estimated for our problem in Section \ref{subsec: param est}.}\vn{Do we say this explicitly?}
% Given these probabilities, we seek to maximize the cumulative probability of successful vaccination given a fixed budget $b$ by finding the optimal allocation of interventions $I$ among the $M$ mothers.
% , which is equivalent to finding an allocation of the interventions in $I$ among these $M$ mothers that maximizes the 
% total sum of probabilities of vaccine intake.
\noindent \textbf{Decision Variables:} Given grid $G$, mothers $M$, and a time horizon of $[T]$ days, we optimize over the allocation of interventions $I$. We use $g$ and $t$ to denote an arbitrary cell in $G$ and an arbitrary day in $[T]$ respectively. Let $x_{tg}$ be a binary variable that denotes the decision to conduct a vaccine drive, i.e., $x_{tg}$ is $1$ if and only if there is a vaccine drive at cell $g$ on day $t$, and $0$ otherwise. We point out that a vaccination drive at a cell does not necessarily target every mother in that cell. A healthcare official can only visit a fixed number of households, and our optimization formulation must optimize which mothers to target during a drive. If possible, the healthcare worker will travel to nearby cells as well.

Let $R_{f}$ denote the set of routes that a bus $f \in F$ can operate (we explain constraints specific to routes later; we first present our optimization formulation here for ease of exposition). We use a binary variable $q_{tfr}$ to denote the routes that are chosen for operation, i.e., $q_{tfr} = 1$ if bus $f$ operates on route $r\in R_{f}$ on day $t\in [T]$. Note that a specific route can only potentially target a subset of the mothers based on their locations. We use binary values $s_{mtfr}$ to denote whether mother $m$ can be picked up by a bus $f\in F$ operating on route $r \in R_f$ on day $t\in [T]$.

% the eligibility of mothers to routes; denoting whether mother $m$ can be picked up by bus $f$ on route $r$; $s_{m,f, r} = 1$ indicates mother $m$ can be picked by bus $f$ on route $r\in R_{f}$. 
We use additional $u, y,$ and $z$ variables to match specific interventions to each mother. The variable $y_{mtj} = 1$ if mother $m$ is given intervention $i_j\in \{i_n, i_c,i_t\}$ at time $t$.
For interventions $i_{\ell}$ and $i_v$, we have variables $u$ and $z$ such that:
a) $u_{mtfr} = 1$ if mother $m$ is picked up by bus $f$ employing route $r \in R_{f}$ on day $t\in [T]$, and
b) $z_{mtg} = 1$ if mother $m$ is targeted on day $t$ by a vaccination drive conducted at cell $g$.

\noindent \textbf{Objective Function:} Formally, we seek to optimize the following objective:
\small
\begin{align}
    &u^*, q^*, x^{*},y^{*},z^{*} =  \argmax_{u,x,y,z} \sum_{m \in [M]} \sum_{i_j \in \{i_n, i_c,i_t \}} y_{mtj} p_{mj} \label{eq:objective}\\
    &+ \sum_{ m \in [M]} \sum_{t \in [T]} \sum_{g \in [G]} z_{mtg} p_{mv}  + \sum_{ m \in [M]} \sum_{t \in [T]} \sum_{f \in [F]}\sum_{r\in R_f} u_{mtfr} p_{m\ell} \nonumber
\end{align}
\normalsize
We seek to maximize the cumulative probability of successful vaccination given a fixed overall budget $b$ by finding the optimal allocation of interventions $I$ among the $M$ mothers.

\noindent \textbf{Constraints:} We need to enforce the following constraints given our problem setting:

\noindent 1. \textit{Eligibility Constraints}: Each mother must be eligible for the vaccine when she is being targeted for an intervention. 
\small
    \[y_{mtj} \leq a_{m t} \;\;\; \forall m \in [M], t \in [T], i_j \in \{i_n, i_c, i_t\} \]
    \[z_{mtg} \leq a_{m t} \;\;\; \forall m \in [M], t \in [T], g \in [G]\]
    \[
    u_{mtfr} \leq a_{m t} \;\;\; \forall m \in [M], t \in [T], f \in [F], r\in R_f
    \]
\normalsize    
\noindent 2. \textit{Vaccine Drive Constraints}: a) If a mother is being targeted for a vaccination drive at a given location and time, there must exist such a drive, b) only mothers that live within distance $\sigma$ of a drive can be targeted for the drive, and c) at most $\gamma_v$ mothers can be targeted by a single drive. The last two constraints denote operational limitations of conducting door-to-door vaccination drives. 
\small
    \[z_{mtg} \leq x_{tg} \;\;\; \forall m \in [M], t \in [T], g \in [G]\]
    \[z_{mtg} \, d_{mg} \leq \sigma \;\;\; \forall m \in [M], t \in [T], g \in [G]\]
    \[\sum_{m \in [M]} z_{mtg}  \leq \gamma_v \;\;\; \forall t \in [T], g \in [G]\]
\normalsize    
%\noindent\textit{3. Location Constraint}: 
%\noindent\textit{4. Vaccine Drive Cap Constraint}: 
    
\noindent 3. \textit{Route Constraints}: If a mother is being being picked by a bus on a route on a particular day, then a) the mother should be eligible to be picked on that route,
b) the bus must employ that route on that day, and c) each bus can pick up at most $\gamma_{\ell}$ mothers, and d) in addition, given current resource limitations of our partner agency, we consider that a bus can only operate a single route on a given day.
\small
    \[u_{mtfr} \leq s_{mtfr} \;\;\; \forall m \in [M], t \in [T], f\in [F], r \in R_{f}\]
    \[u_{mtfr} \leq q_{tfr} \;\;\; \forall m \in [M], t \in [T], f\in [F], r \in R_{f}\]
    \[\sum_{m \in [M]} u_{mtfr} \leq \gamma_{\ell} \;\;\; \forall t \in [T], f\in [F], r \in R_{f}\]
    \[\sum_{r\in R_{f}} q_{tfr} \leq 1 \;\;\; \forall t \in [T], f\in [F]\]
\normalsize 
Note that each vehicle route must obey general routing constraints, e.g., there are restrictions on the earliest pick-up times and the latest drop-off times in our setting. We assume that all routes in $R_f$\,, $\forall f\in F$ obey these constraints (for now) to simplify the discussion (discussed in detail in section~\ref{subsec: method, vrp}).

\noindent 4. \textit{Intervention constraint}: We consider that each mother can be targeted for at most one intervention, i.e. for all $m \in [M]$,
\small
    \begin{align*}
     \sum_{t \in T}& \sum_{i_j \in \{i_c,i_t\}} y_{mtj} + \sum_{t \in T} \sum_{g \in G} z_{mtg} + 
     \sum_{t \in T} \sum_{f \in [F]} \sum_{r\in R_{f}} u_{mtfr} \leq 1
    \end{align*}
\normalsize    
\noindent\textit{5. Budget Constraint}: The total cost of the interventions can not exceed the monetary budget $b$ of the organization.
\small
  \begin{align*}
      \sum_{ m \in M}& \sum_{t \in T} \sum_{i_j \in \{i_c,i_t\}} y_{mtj} \cdot e_j + \sum_{t \in T} \sum_{g \in G} x_{tg} \cdot e_v \\
      & \sum_{t \in T} \sum_{f \in [F]} \sum_{r\in R_{f}} q_{tfr}\cdot e_{\ell}        ~~~~\leq b
  \end{align*}
\normalsize

\subsection{Routing Formulation}\label{subsec: method, vrp}

We formulate a vehicle routing problem with time windows (VRPTW)~\cite{toth2002vehicle} to schedule vehicles to pick up mothers (and their children) and take them to a vaccination center. Vehicle routing problems can be static, where all inputs are received
before optimizing routes, or dynamic, where inputs are updated concurrently with the determination of the route~\cite{pillac2013review,wilbur2022online}. We consider a static VRP; the set of mothers whose children need vaccination on a given day is known before routes are optimized. In practice, the mothers need to be taken to the health centers and dropped back to their resp. residences. However, we only discuss routing to the health centers to simplify the discussion.
% the intervention dictates that mothers are dropped back to their pickup points as well. However, a) our optimization formulation only requires the cumulative probability of success given the beneficiaries, and b) routing to the health centers is more challenging due to constraints on the latest drop-off time. As a result, it suffices to consider the routing problem to the vaccination centers. 
%Dropping mothers to their homes is trivial as we can assume the vehicles take the same routes in reverse. 
% Since a bus can take the same route back, the problem simplifies to picking up a mother within their time window and dropping the mother at a vaccination site within the vaccination site's time window. %We denote the list of vaccination sites 
All vehicles begin operation from fixed spots (parking locations rented by HelpMum) called \textit{depots}. Note that on day $t$, only a subset of mothers are eligible for vaccination, i.e., $a_{mt} = 1$.
%assigned intervention $i_v$ (vaccine drive) which is $M_{t}=\{m \in M, k \in T, i_j \in I \mid i_j=v, a(i,k)=1, k=t\}$. 
Let $\beta_{e}(m)$ and $\beta_{l}(m)$ denote the earliest and latest times on which mother $m$ can be picked up. The times vary across the population based on occupation and other beneficiary specific constraints. Let the set of vaccination centers operating on day $t$ day be $S_{t}$. HelpMum requires that mothers are dropped off at a vaccination center early so that there is sufficient time for them to get their children vaccinated. Let the earliest and latest drop-off times for a vaccination centre $s \in S_t$ be denoted $\beta_{e}(s)$ and $\beta_{l}(s)$ respectively. The set of pick-up locations (mothers' residences) and drop-off locations (vaccination centers) represent the nodes ($N$) of a graph with the road network being the edges.

% Lastly, there are a set of vehicles available to pickup the mothers ($F_t$). The set of vehicles is a subset of the total vehicles, therefore $F_t \subset F$.
% We use $N$ to denote the entire set of locations, the pickup location of each mother and the location of the vaccination sites.

%The set of all possible locations in the operating region is represented by a graph $G=(N,E)$ where $N$ denotes the set of vertices (i.e. locations) and $E$ denotes the set of edges (i.e. roadways between locations) in the graph. Each edge is weighted by the travel time between two locations in the graph. Therefore the cost of travelling between two nodes is the shortest path travel time between those two nodes. 

A route plan is denoted by an ordered sequence of nodes $\theta = \{n_1, n_2, ...\}$, where $n_j \in N$ is an arbitrary node that the vehicle needs to visit en-route. We attach the following information with each node in a route plan. First, $\beta_{\theta,e}(n_k)$ and $\beta_{\theta,l}(n_k)$ is the earliest and latest the vehicle can arrive at location $n_k$, and is set to $\beta_{e}(m)$ and $\beta_{\theta,l}(n_k)$ respectively, when the node corresponds to a pickup location for a mother $m\in M$, or $\beta_{e}(s)$ and $\beta_{l}(s)$ respectively when the node corresponds to a vaccination site $s\in S_t$. 
%Similarly, $\beta_{\theta,l}(n_k)$ is the latest the vehicle can arrive at the location $n_k$ which is set to $\beta_{l}(m)$ when it corresponds to a pickup location for a mother $m\in M_t$, or $\beta_{l}(s)$  when it corresponds to a vaccination site $s\in S_t$. 
Second, the scheduled arrival time for the vehicle servicing route plan $\theta$ to location $n_k$ is $\delta_{\theta}(n_k)$.
%Also, $w(\theta, n_k)$ is a binary variable that is 1 if location $n_k$ is a mother pickup location and $0$ otherwise.
A route plan is feasible if all time window constraints are satisfied and all mothers who are picked up (up to a maximum capacity of each vehicle) are dropped off at a vaccination site. The time window constraints are satisfied for each location if $\beta_{\theta,e}(n_k) \leq \delta_{\theta}(n_k) \leq \beta_{\theta,l}(n_k)$, $\forall n_k \in \theta$. For each vehicle $f \in F$, the set of feasible routes contain all routes that obey all the routing and capacity constraints.

% The second constraint is satisfied when the last location is a vaccination site ($w(\theta, n_{|\theta|})=0$)
%\Ayan{What does $w(\theta, n_{|\theta|})=0$ mean? What is $w$? Let us remove notation if we do not need it. There is a lot of notation in this paper.}
%%%% APPROACH %%%%
\section{Approach}\label{sec: approach}

We face three specific challenges in directly solving optimization problem (~\ref{eq:objective}). First, the ILP consists of more than $10^9$ decision variables and constraints. Second, we must generate the set of feasible routes as an input to the ILP. In our problem setting, the number of routes exceed $10^{200}$. Third, we must estimate the success of each intervention for each mother. We tackle these challenges in a principled manner below. 
% In order to tackle these challenges, we propose a principled heuristic approach called \ALG\ (AI Driven Vaccination Intervention Optimiser) (see Figure~\ref{fig:workflow}). We assume that the probabilities $p_{mj}$ and are known for each mother and intervention for now. We describe the estimation of the probabilities in section~\ref{subsec: param est}.
Our approach (shown in Figure~\ref{fig:helpmumAndWorkflow} (b)) is based on pruning the search space of the decision variables by greedily conducting the most \textit{efficient} intervention; specifically, we greedily use the given budget to conduct the intervention that has the highest success-to-cost ratio. In this case, conducting a vaccination drive is relatively more expensive than making a phone call, but it enables HelpMum to target more mothers and guarantee more successful vaccinations. 
% We describe our greedy approach in Sec. \toref\ and present bounds for the loss with respect to solving the ILP optimally.
After the greedy allocation of the vaccination drives we use guided local search to generate promising vehicle routes, which is fed to the ILP as an input. 

\begin{algorithm}[h]
\caption{The ADVISER Framework}
\label{alg:MAIN ILP}
\textbf{Input} $M$: set of mothers, $a$: availability matrix, $d$ distance matrix, $T$: number of days in program, $G$: grid, $b$: budget, $r$: radius \\
\textbf{Output} $I$: intervention alloc. array
\begin{algorithmic}[1] %[1] enables line numbers
\STATE $M', C, b' \leftarrow \text{Heuristic}(M,a,d,G,b, r)$ \\\textcolor{gray}{/* Calls the Heuristic algorithm (Algorithm \ref{alg: Heuristic}) which returns $M'$: the remaining set of mothers, $C$: the vaccine drive allocation matrix indicating where the vaccine drives will be conducted, and the left over budget $b'$. */}
\STATE $I(m) \leftarrow \text{Vaccine Drive}$, for $m \in M\setminus M'$\\ \textcolor{gray}{/* Assigns the vaccine drive intervention to all mothers in $M$ but not in $M'$. */}\\
\STATE $R \leftarrow VRP(M', \text{Time windows})$ \\ \textcolor{gray}{/* Calls the Vehicle Routing Algorithm (Sections \ref{subsec: method, vrp}, \ref{sec: approach}) with inputs $M'$, time windows for each mother, and immunization centres. It returns $R$: the set of optimal routes for each
(bus depot, centre) pair on each day. */}
\STATE $I(M') \leftarrow \text{ILP}(M', R, a, d, C, G, b', r)$ \\ \textcolor{gray}{/* Calls the ILP (Section \ref{sec: problem form}) on the remaining mothers $M'$ with routes $R$, leftover budget $b'$, and other relevant parameters. The vaccine drive allocation matrix $C$ helps the ILP ensure that there are no vaccination drives at grid $g \in G$ on day $t$, if the Heuristic has already decided to conduct a drive there on that day. The ILP returns the optimal intervention allocation on remaining mothers with budget $b'$. */}
\STATE Return the intervention allocation array $I$.
\end{algorithmic}
\end{algorithm}

\subsection{Greedy Pruning} We assume that the probabilities $p_{mj}$ are known for each mother and intervention. We describe how such parameters can be estimated later. We use an iterative approach for pruning the size of the ILP. At each step, a set of mothers are chosen for intervention and removed from consideration. 
%In the beginning of $w$-th iteration, let $M_w \leq M$ denote the number of mothers left for intervention. Without loss of generality, 
Let $[M_w]$ denote the set of mothers at the beginning of iteration $w$. Let $H^w$ denote a matrix of size $G\times T$, where each entry in the matrix (denoted by $H^{w}_{gt}$) captures the utility of conducting a vaccination drive at cell $g \in [G]$ and time $t \in [T]$. At iteration $w$, let $g^*_w$ on day $t^*_w$ denote the optimal cell and day to conduct an intervention, given $[M_w]$. The matrix $H^w$ is used to choose cell-time combinations to conduct vaccination drives in each iteration. The cell-time positions chosen in the previous iterations are updated to $-1$ in the matrix to remove them from consideration in future iterations, i.e., $g^*_{w'}, t^*_{w'}$ for all $w' < w$ is set to $-1$. We point out that conducting a vaccination drive at a cell can potentially target mothers from nearby cells as well, depending on the number of households in consideration and the manner in which door-to-door vaccine delivery is done. Let $M_{gt}$  denote a subset of mothers in $[M_w]$ who live within some exogenously specified distance $\sigma$ of cell $g$ and are eligible for vaccination at time $t$. $H^w$ can then be computed as:
\small
\[
    H^w_{gt} =  
\begin{cases}
    -1 & \text{if } g=g^*_{w'} \wedge t=t^*_{w'} \ \forall w'<w\\
    U^{w}(g,t)            & \text{otherwise}
\end{cases}
\]
\normalsize
where $U^{w}(g,t)=\max_{S\subseteq M_{gt}, |S|\leq \gamma_v}\sum_{m\in S} (p_{mv} - p_{mn})$ denotes the utility of conducting a vaccine drive at cell $g$ at time $t$ on iteration $w$ by targeting the subset of mothers who provide the most \textit{gain} over no interventions. Note that the utility depends on $w$, because mothers who are targeted for intervention in an iteration are removed from consideration on the next iteration. As an example, conducting a vaccination drive at the same location on two successive days won't be ideal, if every mother eligible for vaccination is already benefited by the service on the first day. Then, $g^*_{w}, t^*_{w} = \argmax_{g,t} H^{w}_{g,t}$. In our formulation, $p_{mv} =1$. We denote the mothers targeted as part of the conducting a drive at $g^*_{w}, t^*_{w}$ as $S^{w}_{g^* t^*} = \argmax_{S\subseteq M_{g^*t^*}, |S|\leq \gamma_v}\sum_{m\in S} (p_{mv} - p_{mn})$.

We drop references to $w$ to simplify the discussion. In the proposed heuristic, we decide to conduct a vaccine drive at $g^*, t^*$ if 
%\[
$e_t\cdot |S_{g^* t^*}| \geq e_v$, i.e.,
%\]
if the cost of conducting a vaccine drive is at most the cost of giving travel vouchers to the mothers being benefited by the drive. Note that conducting a vaccine drive is a better intervention than providing travel vouchers even when they cost the same (as children are guaranteed to be vaccinated through the former strategy). However, our pruning strategy is \textit{lazy}; we commit to conducting a drive only if the aforementioned condition is satisfied and leave other decisions for the pruned ILP. The mothers who are mapped to vaccine drives are removed from consideration at the next iteration and the budget is updated accordingly. We stop pruning after we are left with some exogenously specified budget parameter ($b'$). We present the pseudo codes for the ADVISER framework in Algorithm~\ref{alg:MAIN ILP} and the heuristic pruning procedure that ADVISER uses in Algorithm~\ref{alg: Heuristic}.

\begin{algorithm}[h]
\caption{Heuristic Pruning}
\label{alg: Heuristic}
\textbf{Input} $M$: set of mothers, $a$: availability matrix, $d$: distance matrix, $T$: number of days in program, $G$: grid, $b$: budget, $r$: radius \\
\textbf{Output} $M'$: set of remaining mothers, $C$: vaccine drive allocation matrix, $b'$: leftover budget 
\begin{algorithmic}[1] %[1] enables line numbers
\STATE Let $C, K, H$ be matrices of size $|G|\times T$. \\
\textcolor{gray}{/* $C_{gt} = 1$ if the Heuristic decides to conduct a vaccine drive at grid $g$ on day $t$ and otherwise $0$, $K_{gt} = 1$ if at step 10 $g, t$ is computed as $g^*, t^*$ and $0$ otherwise, $H_{gt}$ is the value of computing the vaccine drive at cell $g,t$ (refer Section \ref{sec: approach}) */}
\STATE Initialize: $C_{gt} = 0$, $H_{gt} = 0$, and $K_{gt} = 0$ for all $g\in G$ and $t \in [T]$
\STATE Set Count $\leftarrow 0$
%\at{shudnt this be while b>=b'}
\WHILE{$b \geq b'$ and $\text{count} \leq |G|\times T$}
\STATE Set Count $\leftarrow$ Count + 1
\FOR{$g\in [|G|]$ and $t\in [T]$} 
\STATE $M_{gt} \leftarrow \{m\in M \mid a_{mt} = 1, d_{mg} \leq r \}$
\IF{$K_{gt} = 1$}
\STATE $H_{gt} \leftarrow -1$
\ELSIF{$K_{gt} = 0$}
\STATE $H_{gt} = \max_{S\subseteq M_{gt}, |S|\leq \gamma_v}\sum_{m\in S} (p_{mv} - p_{mn})$\\
\ENDIF
\ENDFOR
\STATE $g^*, t^* = \argmax_{g,t} H_{gt}$
\STATE $S_{g^* t^*} = \argmax_{S\subseteq M_{g^*t^*}, |S|\leq \gamma_v}\sum_{m\in S} (p_{mv} - p_{mn})$
\STATE $K_{g^* t^*} \leftarrow 1$
\IF{$e_t\cdot |S_{g^* t^*}| \geq e_v$}
\STATE $C_{g^* t^*} \leftarrow 1$
\STATE $M \leftarrow M\setminus S_{g^* t^*}$
\STATE $b \leftarrow b-  e_v$
\ENDIF
\ENDWHILE
\STATE $b' \leftarrow b$
\STATE Return $M'$, $C$, $b'$
\end{algorithmic}
\end{algorithm}

\textbf{Performance Bounds:} Let $k$ denote the number of vaccination drives determined in the greedy pruning phase, and $M_{VH}$ denote the set of mothers targeted by these $k$ vaccination drives. We bound the loss incurred through greedy pruning relative to the optimal solution of the ILP. We assume that the optimal solution has at least $k$ vaccination drives of size at least $e_v/e_t$; we verify this empirically in multiple parameter settings. We arbitrarily choose $k$ vaccine drives from such an optimal solution, and denote by $M_{VI}$ the set of mothers targeted by these $k$ vaccine drives. Also, let mother $m\in M$ be given intervention $i_m^*$ by such an optimal solution. 
We begin by proving the following proposition, which is an outcome of our greedy choice made at every iteration during pruning. 
%Proposition \ref{prop: objective for greedy phase} is used to 
% prove Theorem \ref{thm: lower bound loss} below.
\begin{restatable}{proposition}{objgreedy}\label{prop: objective for greedy phase}
$\sum_{m \in M_{VH}} (p_{mv} - p_{mn}) \geq \sum_{m \in M_{VI}} (p_{mv} - p_{mn})$
\end{restatable}
\begin{proof}
Let $v_{wH}$ be the vaccine drive determined by the heuristic procedure at the $w$-th iteration for $w \in [1,k]$. Also, arbitrarily order the $k$ vaccine drives that the mothers in $M_{VI}$ are part of. In particular, let $v_{wI}$ be the $w$-th vaccine drive of such a chosen order among the $k$ vaccine drives that are part of the optimal ILP solution, for $w\in [1,k]$. Further, let $M_{VH}^{(w)}$ be the set of mothers targeted by the $w$ vaccination drives $v_{1H}, \ldots, v_{wH}$. Similarly, let $M_{VI}^{(w)}$ be the set of mothers targeted by the $w$ vaccination drives $v_{1I}, \ldots, v_{wI}$. Since at every iteration $w' \leq w$ in the greedy pruning we choose to have a vaccine drive at grid $g^*_{w'}$ and time $t^*_{w'}$ with the highest value of $H^{w'}_{gt}$, we have for every $w \in [1,k]$
$$\sum_{m \in M_{VH}^{(w)}} p_{mv} - p_{mn} \geq \sum_{m \in M_{VI}^{(w)}} p_{mv} - p_{mn} \ .$$
Finally, it is easily seen that $M_{VH}^{(k)} = M_{VH}$ and $M_{VI}^{(k)} = M_{VI}$, and as the above equation holds for $w =k$, this completes the proof of the proposition.
\end{proof}

\noindent Let $M_{VH} \setminus M_{VI}$ (resp. $M_{VI} \setminus M_{VH}$) denote the set of mothers in $M_{VH}$ (resp. $M_{VI}$) but not in $M_{VI}$ (resp. $M_{VH}$). We use Proposition 1 to prove the following theorem.
%In the following theorem we determine how far the objective value derived from our heuristic procedure is from the optimal objective value. 
\begin{restatable}{theorem}{lowerbound}\label{thm: lower bound loss}
Let $O_H$ be the objective value of the solution derived from our heuristic procedure and $O^*$ be the objective value of the optimal ILP solution. Then $O_H \geq O^* - (\sum_{m \in M_{VH}\setminus M_{VI}} (p_{m i_m^*} - p_{mn}))$.
\end{restatable}
\begin{proof}
From Proposition \ref{prop: objective for greedy phase}, we have
\begin{align*}
    \sum_{m \in M_{VH}} p_{mv} - p_{mn} &\geq \sum_{m \in M_{VI}} p_{mv} - p_{mn} 
\end{align*}
Rearranging the above expression we have
\begin{align*}
    \sum_{m \in M_{VH}} p_{mv} + \sum_{m \in M_{VI}\setminus M_{VH}} p_{mn}  \geq \\ \sum_{m \in M_{VI}} p_{mv} +  \sum_{m \in M_{VH}\setminus M_{VI}} p_{mn}
\end{align*}
We add $\sum_{m \in M\setminus M_{VI}} p_{mi_m^*} $ to both sides of the above expression.
\begin{align*}
    \sum_{m \in M_{VH}} p_{mv} + \sum_{m \in M_{VI}\setminus M_{VH}} p_{mn} \ + \sum_{m \in M\setminus M_{VI}} p_{mi_m^*}  \geq \\ \sum_{m \in M_{VI}} p_{mv} +  \sum_{m \in M_{VH}\setminus M_{VI}} p_{mn} +\sum_{m \in M\setminus M_{VI}} p_{mi_m^*} 
\end{align*}
Now observe that $O^* = \sum_{m \in M_{VI}} p_{mv} +\sum_{m \in M\setminus M_{VI}} p_{mi_m^*}$, and $M \setminus M_{VI} = (M_{VH}\setminus M_{VI}) \ \uplus \ (M\setminus (M_{VH}\cup M_{VI}))$. Hence, substituting for $O^*$ in RHS  and partitioning $M\setminus M_{VI}$ in the LHS of the above expression we have
\begin{align*}
    &\sum_{m \in M_{VH}} p_{mv} + \sum_{m \in M_{VI}\setminus M_{VH}} p_{mn} + \sum_{m \in M_{VH}\setminus M_{VI}} p_{mi_m^*} \\ &+ \sum_{m \in M\setminus (M_{VH}\cup M_{VI})} p_{mi_m^*}  \, \, \geq \,  O^* + \sum_{m \in M_{VH}\setminus M_{VI}} p_{mn} 
\end{align*}
Rearranging the above expression, we have
\begin{align}\label{equation: thm 1 the last inequality}
    &\sum_{m \in M_{VH}} p_{mv} + \sum_{m \in M_{VI}\setminus M_{VH}} p_{mn} + \sum_{m \in M\setminus  (M_{VH}\cup M_{VI})}  p_{mi_m^*}  \nonumber \\ & \geq  O^* - (\sum_{m \in M_{VH}\setminus M_{VI}} p_{mi_m^*} - p_{mn})
\end{align}
%\normalsize
It is easy to see that $M \setminus M_{VH} = (M_{VI}\setminus M_{VH}) \ \uplus \ (M\setminus (M_{VH}\cup M_{VI}))$, and recall that ADVISER finally runs the ILP on $M\setminus M_{VH}$ mothers with budget $b-k\cdot e_v$. Now since the number of vaccine drives in the optimal solution is at least $k$, the cost of providing interventions $i_m^*$ to mothers $m\in  M\setminus (M_{VH}\cup M_{VI})$ is at most $b - k\cdot e_v$.
%(\textcolor{red}{it is here that we require the assumption that the optimal solution has at least $k$ vaccine drives}). 
Hence, providing no interventions to mothers in $M_{VI}\setminus M_{VH}$ and intervention $i_m^*$ to mother $m \in M\setminus (M_{VH}\cup M_{VI})$ is a feasible solution of the ILP run on the remaining mothers. This implies the ILP on the remaining mothers returns an intervention allocation on $M\setminus M_{VH}$ which has objective value at least $\sum_{m \in M_{VI}\setminus M_{VH}} p_{mn} + \sum_{m \in M\setminus (M_{VH}\cup M_{VI})} p_{mi_m*}$. Hence the objective value of the heuristic procedure is

\begin{align}\label{equation: lower bound on O_H}
 &O_H \geq \sum_{m \in M_{VH}} p_{mv} + \sum_{m \in M_{VI}\setminus M_{VH}} p_{mn} \nonumber \\& + \sum_{m \in M\setminus (M_{VH}\cup M_{VI})} p_{mi_m*}   
\end{align}
Finally, using Equation \ref{equation: lower bound on O_H} in Equation \ref{equation: thm 1 the last inequality} we have
$$O_H \geq O^* - (\sum_{m \in M_{VH}\setminus M_{VI}} p_{mi_m^*} - p_{mn}) \ .$$
\end{proof}
\noindent The lower bound on the objective value of the heuristic approach in Theorem~\ref{thm: lower bound loss} depends on the interventions provided in the optimal solution to mothers in $M_{VH}$ but not in $M_{VI}$.

\begin{comment}
Then  Then the our greedy pruning phase ensures \textcolor{red}{ to prove}

$$\sum_{i \in M_{VH}} p_{iv} - p_{i,ni} \geq \sum_{i \in M_{V,I}} p_{i,v} - p_{i,ni}$$
Similarly, let $M_{TV,H}$ and $M_{TV,I}$ ($M_{TV,H}$ and $M_{TV,I}$)
Then we also assume that for all $i\in M_V$ and $j\in M$

terminates after determining the $k$ vaccine drives, and that the optimal integer linear program solution also has $k$ vaccine drives. 
\end{comment}

\subsection{Route Generation}
In principle, we could generate all the feasible routes given the routing constraints, which can then be provided as an input to the optimization problem. However, route generation is intractable in our setting. The total number of routes in our setting exceeds $10^{200}$. As a result, we focus on generating a smaller subset of \textit{promising} routes. Recall that our overall goal is to maximize the cumulative probability of successful vaccination; as a result, it is imperative that given the limited number of vehicles, we pickup mothers that need the ride the most. We capture this idea to define the utility of a route plan. Let $p(\theta, n_k) = p_{m\ell} - p_{mn}$ denote the utility of an arbitrary pick-up node $n_k \in N$ in the routing graph (recall that each pickup node corresponds to a unique mother). The quantity $p_{m\ell} - p_{mn}$ captures the importance of providing a bus pickup to mother $m$ over giving no intervention. Given the utility function, we use guided local search~\cite{kilby1999guided} to generate a subset of routes that maximize the utility.
% \usepackage{booktabs}
\begin{table*}[t]
\caption{Description of the features used to learn the probability of success for the interventions}
\centering
\begin{tabular}{@{}lll@{}}
\toprule
Feature            & Type    & Description                                                                                                                                                 \\ \midrule
Vaccination Status & Binary  & A variable denoting whether the mother took her child for vaccination                                                                                       \\
Income Level       & Binary  & A variable denoting whether the family earns more than \$25 or not.                                                                                         \\
Message Status     & Binary  & \begin{tabular}[c]{@{}l@{}}A binary variable that denotes whether the mother received a message \\ about the upcoming vaccination appointment.\end{tabular} \\
Age of the mother  & Integer & The age of the mother in years.                                                                                                                             \\
Age of the child   & Integer & The age of the child in months.                                                                                                                             \\
Number of children & Integer & The number of children the mother has.                                                                                                                      \\
Address            & String  & The neighborhood that the mother lives in.                                                                                                                  \\
Vaccination Center & String  & The address of the nearest vaccination center from the mother's house.                                                                                      \\ \bottomrule
\end{tabular}
\label{tab:data}
\end{table*}

\subsection{Parameter Estimation} Note that optimization problem (\ref{eq:objective}) requires estimates of the probability of success of each intervention for each mother.
However, estimating the probabilities presents a challenge---the interventions --- conducting vaccine drives, operating vehicle routes, and providing travel vouchers, are designed as part of this research; as a result, we lack exact historical data about the interventions. We only have data about phone calls that HelpMum made to all mothers. We compute the probability of success of untested interventions (i.e., vaccine drives, bus pickups, and travel vouchers) through a community survey that HelpMum performed. We estimate the probability of successful vaccination through phone calls and the effect of no interventions by learning a regression model on historical data.

We assume that a set of features $W$ can be used to represent each individual. The set $W$ can encode prior information about interventions, income levels, and geographic location. However, estimating the probabilities presents a challenge---the interventions of conducting vaccine drives, operating vehicle routes, and providing travel vouchers are designed as part of this research; as a result, we lack exact historical data about the interventions. Our partner agency reached out to the community to gather feedback about the potential benefit of such interventions. We use feedback from the community outreach to compute the probability of their success. Our partner agency does have data on its routine phone call operation, as part of which it calls every mother to remind them about upcoming vaccination. We use the historical data about phone calls to estimate the success of making additional phone calls by training a logistic regression model.

\subsubsection{Computing the effect of untested interventions}
%\at{This can be made signficantly shorter to save space}
%\at{ The title is misleading, we arent really estimating these values. Maybe Model parameters for previously untested interventions?}
%As part of its current operation, our partner agency reaches out to mothers to remind them about upcoming vaccination appointments and create awareness about the importance of vaccination. 
To aid the estimation of probability of success for interventions that have not been tested, i.e., travel vouchers, bus pickups, and vaccination drives, our partner agency asked the beneficiaries for feedback. All mothers reported that they would welcome healthcare officials when they conduct door-to-door vaccination campaigns. They also reported that transportation costs were a major barrier for accessibility to health centers, and that pickup service or travel vouchers will be of immense value. Based on the feedback, we assume that the probability of a successful vaccination
%if a healthcare official visits a specific household or if a mother (and her child) are picked up by a van are equal to 1. 
for a mother given vaccine drive or picked up by a van equal to 1. Our partner agency reported that in practice, the efficacy of travel vouchers is slightly lower as the intervention lacks direct monitoring (for example, the travel voucher might not be used or can be used for some other purpose). As a result, for the assignment of a travel voucher, we consider the probability of success to be lower than 1, but higher than the probability of success through a phone call alone or the probability of success in the absence of any interventions\footnote{In practice, we anticipate the probability of success to be slightly lower for bus pickups as well. Our partner agency reports that in practice, a health center can exhaust its stock of vaccines. Our estimates can be improved as data is generated through deployment.}. 

We want to estimate the probability that a mother takes a child for vaccination given the intervention of making additional phone calls to remind her about an upcoming vaccination (``additional'' phone calls refer to targeted calls made after all mothers have been called, which our partner agency already does). Estimating the effect of phone calls is somewhat different than the other interventions; we do have historical data from phone calls made by our partner agency. However, note that the probability we seek to estimate is different from $P(\text{mother going to vaccination} \mid \text{phone call is made})$ as we do not want to estimate the empirical conditional probability by restricting attention to the sub-population for which phone calls were made; rather, the phone call is an \textit{intervention}, meaning that we perform the action of making phone calls, which in turn fixes the value of the random variable. Formally, we are interested in estimating $P(\text{mother going for vaccination} \mid do(\text{phone call}))$. However, we point out that a) the marginal distribution of $W$ is invariant under the intervention of making phone calls, and b) the manner in which an individual reacts to a phone call regarding vaccination uptake is the same irrespective of whether the action of making a phone call is through a targeted intervention or not. As a result, the probability of success given the intervention can be directly estimated from historical data by simply calculating the empirical conditional distribution $P(\text{mother going to vaccination} \mid \text{phone call is made})$.

%%%% RESULTS %%%%
\section{Experiments}\label{subsec: exp data}

\subsection{Data} We collect anonymous information from HelpMum for 500 mothers registered as part of a vaccination tracking system operated by HelpMum. Each data point consists of several features such as the income level of the family, whether the mother received a reminder about the upcoming vaccination appointment, whether she took her child for vaccination, and the age of her child, among others. HelpMum obtained consent from each beneficiary for anonymous data sharing. A description of the features we collect is presented in Table~\ref{tab:data}. We also collect the geographic locations of all 32 vaccination centers in our area of interest. The locations of the rented parking depots and the vaccination sites are shown in Figure~\ref{fig:sites}.

HelpMum plans to deploy the ADVISER framework to all mothers in their vaccination tracking system (about 40000 mothers). As a result, we use the available data to generate two synthetic datasets (D1 and D2), consisting of $40000$ mothers each. We generate features of mothers in D1 and D2 by sampling each feature independently and uniformly at random from the original data of $500$ mothers. For the mothers in $D1$, we compute the probabilities of success for each mother given an intervention as follows: for each $m \in [M]$, we choose $p_{mn}$ uniformly at random from $(0,1)$; followed by $p_{mc}$ uniformly at random from $(p_{mn}, 1)$; $p_{mt}$ uniformly at random from $(p_{mc}, 1)$; and $p_{m\ell}$ uniformly at random from $(p_{mt}, 1)$. We set $p_{mv} = 1$ based on community feedback (see section~\ref{secapp: data}). D1 essentially captures the domain knowledge that we have about the interventions, specifically, for a mother $m$, $p_{mn} \leq p_{mc} \leq p_{mt} \leq p_{m\ell} \leq p_{mv} = 1$. For D2, we estimate the probability of vaccination given no intervention (and vaccination given phone calls) by training a logistic regression model on the original data (details in Appendix~\ref{secapp: data}). The probabilities for remaining interventions, $p_{mt}$, $p_{m\ell}$, and $p_{mv}$ are chosen in a similar manner as in D1.

\begin{figure}[h]
    \centering
    \includegraphics[height=5cm]{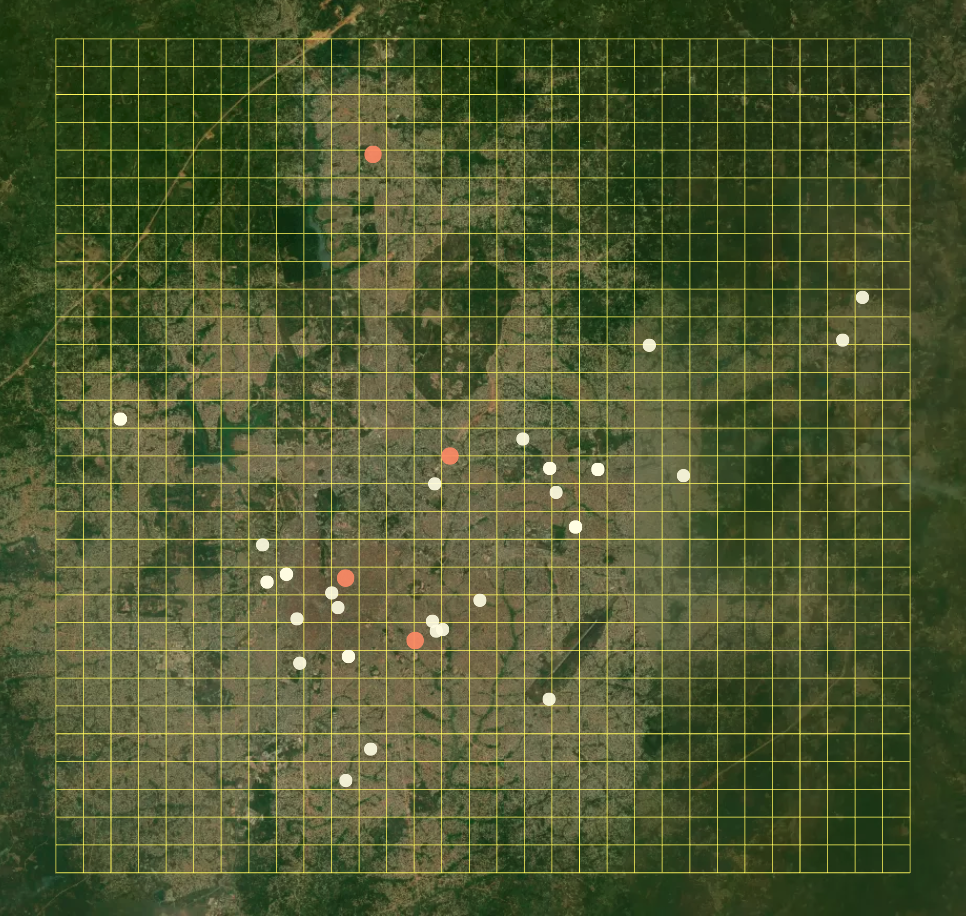}
    \caption{Locations of the rented parking locations (in orange) and the vaccination centers (in white). The yellow lines represent the grid $G$. We see that there the distribution of the vaccination centers is not uniform; however, HelpMum chose to rent a parking location towards the north of the city to ensure that mothers have access to vaccination centers.}
    \label{fig:sites}
\end{figure}

\subsection{Baseline Algorithms:} While prior work does not consist of approaches that optimize the allocation of heterogeneous health resources under uncertainty, we consider the following baselines: 
%We evaluate the performance of our ILP based solution by contrasting it to a round-robin baseline used by the HelpMum. The baseline occurs in four sequential stages, described as follows:

\begin{figure*}[t]
\centering
\includegraphics[height=5cm]{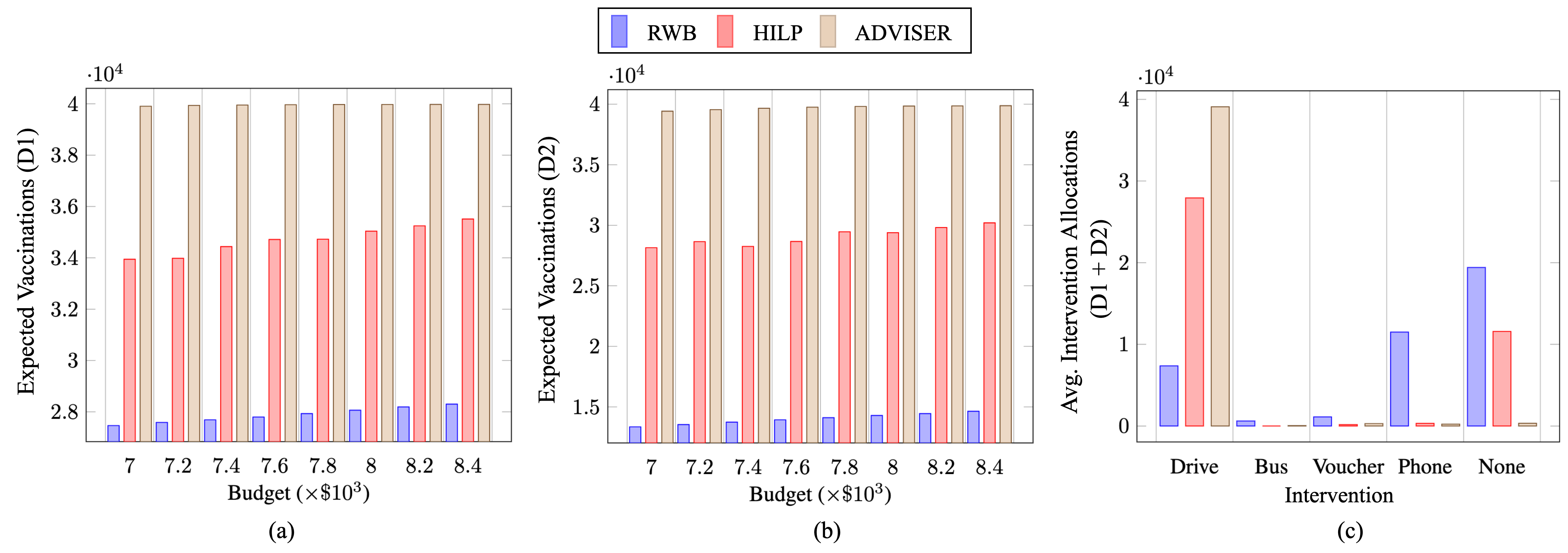}
\caption{(a) Expected number of vaccinations in D1. (b) Expected number of vaccinations in D2. (c) The distribution of the interventions through the different algorithms. We observe that ADVISER outperforms the other two baselines. Also, both ADVISER and HILP choose vaccination drives as the dominant intervention; RWB's poor performance can be explained by having fixed allocations for each intervention.}
\label{fig:combinedResult}
\end{figure*}

\begin{figure*}[h]
\centering
\includegraphics[height=5cm]{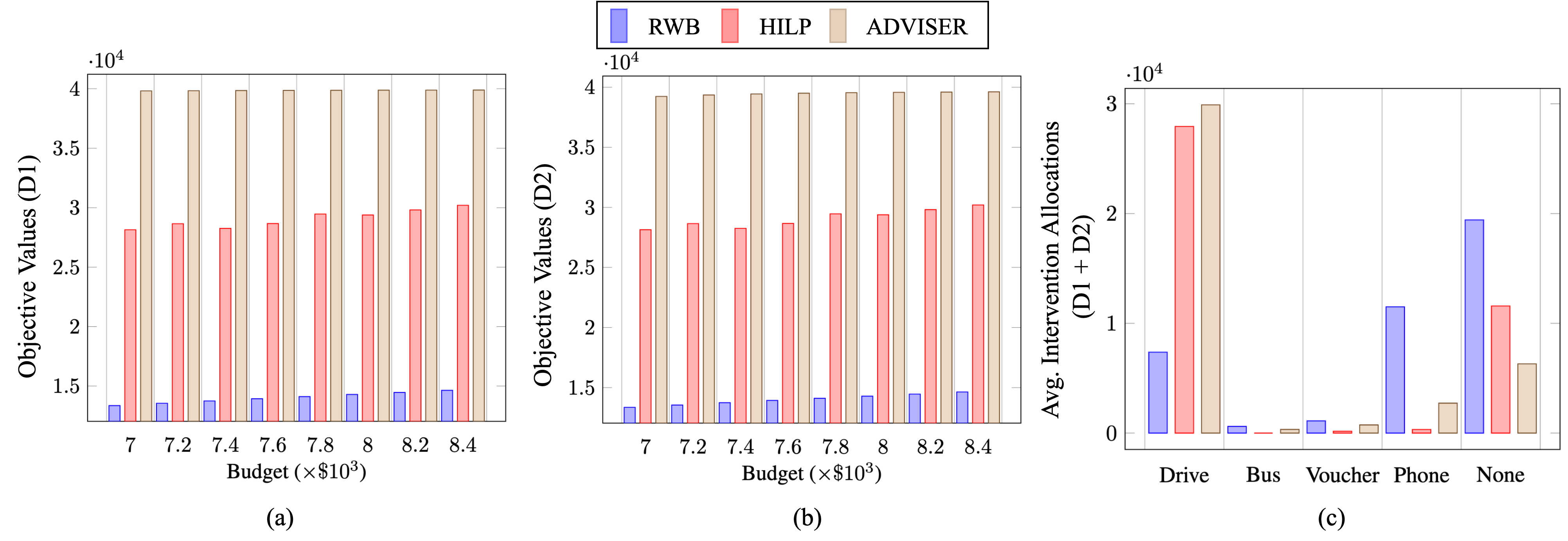}
\caption{The output of ADVISER with number of vaccination drives capped to 400. (a) Objective values for D1. (b) Objective values for D2. (b) Number of interventions allocated averaged across D1 and D2.}
\label{fig:combinedObjAllocNV}
\end{figure*}

\textbf{Real-world Baseline}: For the first baseline, we asked HelpMum to allocate the interventions solely based on domain expertise. HelpMum identified 33 fixed neighbourhoods, one in each local government (similar to administrative jurisdictions) to conduct vaccination drives on alternate days. Bus routes are operated each day using all the $F$ vehicles from the existing depots. Each vehicle serves one vaccination center each day in a round-robin manner. For each trip, mothers who are within some predefined distance from the routes are considered. Note that since the routes determined by our partner agency is fixed, mothers must walk to the bus route; the predefined distance is a check on the distance that a mother can walk with a child to get on a bus. Then, travel vouchers are distributed to mothers who live more than 10 kms. away from a vaccination center. The vouchers are distributed according to income levels, i.e., mothers who have relatively lower income are targeted first. Finally, the remaining budget is used to make targeted phone calls. Our partner agency decided to target mothers based on the age of their child---the younger the child, the higher the priority. This decision is motivated by the fact that an infant requires more vaccination doses, and missing one dose hampers the schedule of upcoming vaccine doses.

\textbf{Hierarchical Integer Linear Programming (HILP)}: Motivated by the use of hierarchical planning to create tractable approaches for resource allocation ~\cite{zhang2016using,pettet2021hierarchical}, we design a baseline that solves optimization problem (~\ref{eq:objective}) in a hierarchical manner. First, we leverage the geographic density of the beneficiaries to identify clusters (using $k$-means~\cite{macqueen1967some}). The overall budget is distributed across clusters in proportion to the number of mothers in each cluster. A separate ILP is then solved directly for each cluster. We start by dividing the entire set of mothers into different clusters via k-means clustering on their geographic locations. Our goal is to create smaller ILP formulations per cluster. Naturally, the abstraction introduced by hierarchical planning also induces a trade-off between scalability and utility. In order to select the optimal number of clusters, we use the \textit{elbow method}~\cite{bholowalia2014ebk} based on the inertia of the clusters (the sum of squared distances of samples to their closest cluster center). The clusters were initialized by sampling the locations of the mothers uniformly at random. The number of clusters for datasets D1 and D2 are 35 and 30 respectively (see Figures \ref{fig: elbow D1} and \ref{fig: elbow D2}).

\subsection{Experiment Setup} 
In consultation with HelpMum, we set the costs as follows: $e_c = \$0.1, e_t =\$1.1$, $e_v = \$15$, and $e_\ell = \$20$. We optimize the allocation of resources for $T=30$ days and $\gamma_v=100$. We vary the overall budget $b$ between $\$7000$ to $\$8400$, and use $b-\$1000$ as a threshold for the greedy pruning procedure. 
% \vn{We do not always use $b -1000$ for pruning. We set the cut off as $b' = 1000$ but we do not always use it.}
Our implementation is available at \small \texttt{\url{https://anonymous.4open.science/r/IJCAI_42/}}\normalsize. All experiments were run on a Linux machine with 64GB RAM and an 8-core AMD processor. We implement ADVISER using the Python programming language and solve the ILP using Google OR Tools with SCIP solver. %(Solving Constraint Integer Programs) 

\subsection{Results} 

\begin{figure*}[]
\centering
\includegraphics[height=5cm]{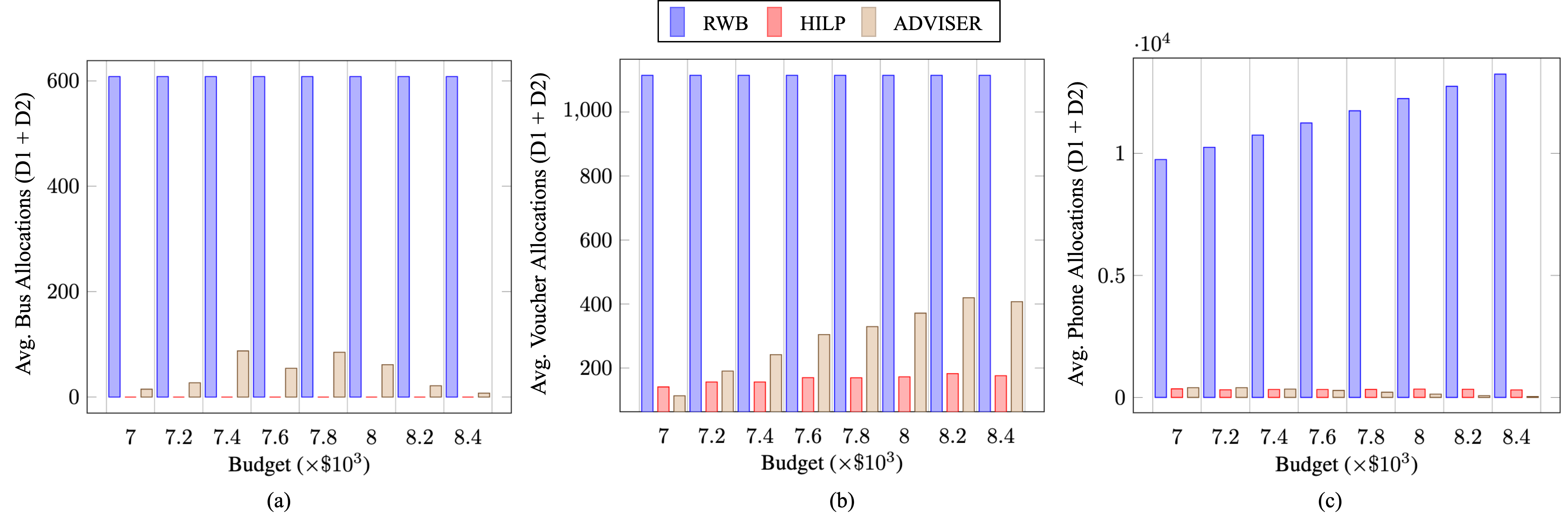}
\caption{Average Intervention Allocations (D1 + D2) for (a) Vehicle Routes, (b) Travel Vouchers, and (c) Phone Calls}
\label{fig:ADVISER_alloc}
\end{figure*}

We show the expected number of successful vaccinations in D1 and D2 in Figure~\ref{fig:combinedResult} (a) and Figure~\ref{fig:combinedResult} (b) respectively.
We observe that the expected vaccine uptake achieved via ADVISER is more than $39970$ for all the budgets considered in the experiment,
whereas the average vaccine intake achieved by the baseline algorithm is at most $26000$ (we performed the simulation on $40000$ mothers). The average number of mothers who received intervention through ADVISER is $39672$, in comparison to $28426$ and $20588$ through HILP and RWB respectively.

We also observe in Figure~\ref{fig:combinedResult} (c) the distribution of the interventions; as expected, both the ILP-based approaches capitalize on solutions with more vaccination drives. However, we point out the importance of the other interventions as well. In practice, the number of vaccination drives is bound by the number of available healthcare workers for the service, whose regular job is to work at healthcare centers. We observed that when the number of vaccination drives is restricted to 400 per month, the average number of mothers who are targeted for pickups more than triples as compared to Figure~\ref{fig:combinedResult} (c) (result presented in appendix~\ref{secapp: results}). Moreover, HelpMum seeks to utilize ADVISER to improve antenatal care for pregnant mothers as well, which will lower the realization of $\gamma_v$ (number of children who can be vaccinated during a drive). Finally, we point out that it is crucial that the ADVISER framework is tractable. On average, the computational time taken to generate solutions by the ADVISER framework is $254$ seconds, as opposed to $4386$ seconds by HILP (RWB can generate solutions in about 10 seconds on average).

\begin{figure}[]%
    \centering
    \includegraphics[width = 0.75\columnwidth]{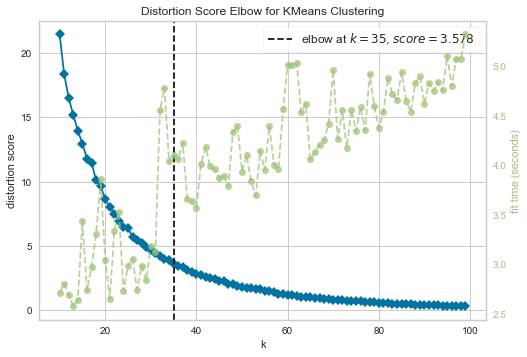}%
    %\subfloat[\centering]{{\includegraphics[width=10cm]{imgs/HelpMum_Workflow.png}}}%
    \caption{Elbow Curve for D1 with $x$ axis representing the number of clusters and $y$ axis representing the distortion score}%
    \label{fig: elbow D1}%
\end{figure}

In order to evaluate the performance of all the approaches under different parameters, we repeat the experiments by capping the maximum number of vaccination drives to 400, i.e., about 13 vaccination drives each day. Essentially, we want to test the robustness of the approaches when sufficient healthcare workers are not available to perform door-to-door vaccination delivery. We show the results in Figure~\ref{fig:combinedObjAllocNV}. We observe that in comparison to the our original setting (shown in the main body of the paper in Figure~\ref{fig:combinedResult}, the number of mothers picked up by the bus service more than triples. This observation highlights the need for a heterogeneous set of interventions. Also, the objective value attained by capping the number of vaccination drives is lower than without the existence of such a bound; however, ADVISER significantly outperforms the baseline approaches in both the settings.

We also show the average distribution of travel vouchers, bus pickups, and phone calls in Figure~\ref{fig:ADVISER_alloc}. We observe that RWB, due to its fixed nature of resource allocation, results in the distribution of a large number of travel vouchers and bus pickups in comparison to ADVISER and HILP (Figure~\ref{fig:ADVISER_alloc} (a)). While travel vouchers and bus pickups are effective modes of intervention, they are relatively more expensive than conducting vaccination drives. The ILP-based approaches (ADVISER and HELP) search the decision-space better an only use such interventions where (intuitively) conducting vaccination drives is not feasible.

\begin{figure}[h!]%
    \centering
    \includegraphics[width = 0.75\columnwidth]{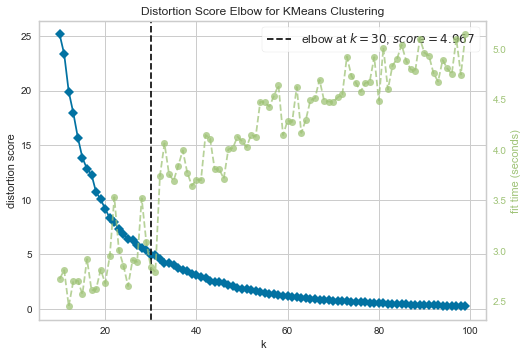}%
    %\subfloat[\centering]{{\includegraphics[width=10cm]{imgs/HelpMum_Workflow.png}}}%
    \caption{Elbow Curve for D2 with $x$ axis representing the number of clusters and $y$ axis representing the distortion score}%
    \label{fig: elbow D2}%
\end{figure}

\subsection{Deployment:} HelpMum is currently planning a pilot program in Ibadan, the largest city in West Africa. We show an initial version of the tool based on the ADVISER framework in Figure~\ref{fig:dashboard}.

\begin{figure}[h]
    \centering
    \includegraphics[width=0.9\columnwidth]{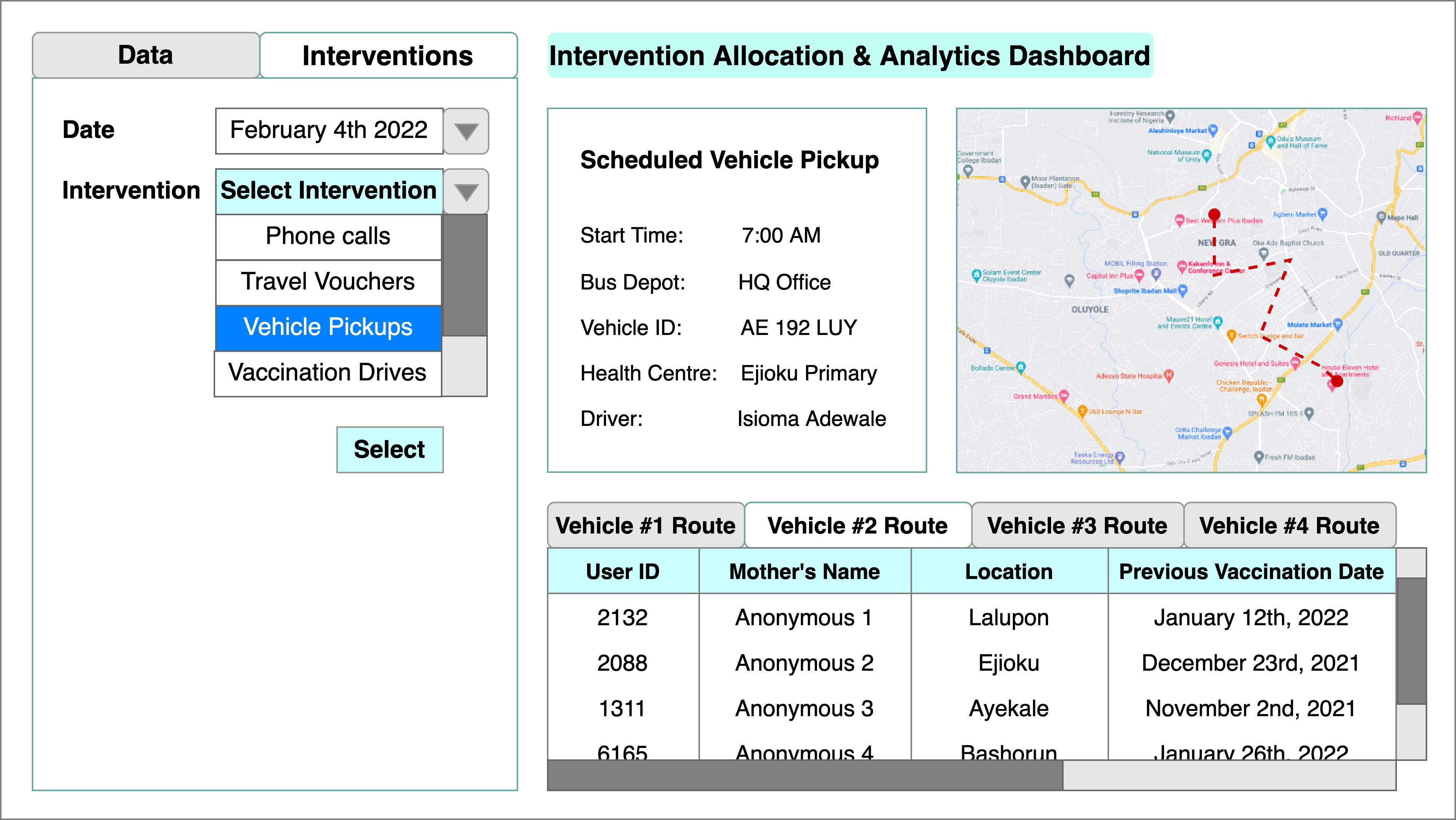}
    \caption{An initial version of the tool that HelpMum will use for deployment. The tool is under construction.}
    \label{fig:dashboard}
\end{figure}

\section{Conclusion}

In collaboration with HelpMum, a non-profit organization in Nigeria, we present ADVISER: AI Driven Vaccination Intervention OptimiSER. Our framework can accelerate our progress towards goals in SDG 1 and SDG 3 by increasing access to healthcare services and vaccination, and by reducing maternal and infant mortality in resource-constrained settings. HelpMum is currently planning a pilot of the ADVISER framework in collaboration with local governments, which will be the first of its kind in Nigeria. 
\section{Ethical Statement}

Our goal in this project is to improve vaccination uptake in Nigeria. We point out that our partner agency, NGO, works closely with state and local governments. The specific interventions were designed by NGO in collaboration with domain experts and its advisory board. NGO is currently planning to deploy the ADVISER framework in the largest city of Nigeria; however, NGO plans to do this in collaboration with the local governments. The fairness of the solution quality of any AI driven framework for public intervention needs to be studied.  While a large part of the interventions suggested by ADVISER are targeted towards low-income individuals (by construction), the objective function of our framework can be modified to add a score function that measures fairness of allocation. Also, it is possible to add arbitrary constraints on the number of interventions across groups in our problem (based on geographic locations). However, the very nature and form of such constraints needs to be determined in collaboration with NGO and local governments.

\section{Acknowledgements}

We would like to acknowledge funding from the Google AI for Social Good Grant, Google AI for Social Good ``Impact Scholars'' program, National Science Foundation grant CNS-1952011, HelpMum, and Vanderbilt University. We would like to thank the staff at HelpMum for helping conduct community surveys and perform data collection.

%\section{Technical Appendix}\label{tech_app}
\bibliographystyle{named}
\bibliography{references}

% \clearpage
% \appendix
% \input{appendix}

\end{document}